\documentclass[runningheads]{llncs}

\usepackage{graphicx}
\usepackage{animate}
\usepackage{subfig}

\usepackage{comment}
\usepackage{color, xcolor}

\usepackage{amsthm}
\usepackage{xcolor}
\usepackage{listings}
\usepackage{media9}
\usepackage{lipsum}   %
\usepackage[frozencache,cachedir=.]{minted}

\usepackage{epsfig}

 \usepackage{microtype} 

\usepackage{chngcntr}

\usepackage{tabularx}
\usepackage{adjustbox}
\usepackage{array}
\usepackage{booktabs}
\usepackage{colortbl}
\usepackage{makecell}
\usepackage{float,wrapfig}
\usepackage{hhline}
\usepackage{multirow}

\usepackage{amsmath,amsfonts,amssymb} %
\usepackage{bm}
\usepackage{nicefrac}
\usepackage{microtype}
\usepackage{dsfont}
\usepackage{changepage}
\usepackage{extramarks}
\usepackage{fancyhdr}
\usepackage{lastpage}
\usepackage{setspace}
\usepackage{soul}
\usepackage{xspace}

\usepackage[pagebackref=true,breaklinks=true,colorlinks,citecolor=black,urlcolor=blue,bookmarks=false]{hyperref}

\usepackage{paralist}
\usepackage{enumitem}

\usepackage{amsmath,amsfonts,bm}

\def\x{{x}}

\def\xi{{\x_i}}

\newcommand{\reffig}[1]{Figure~\ref{fig:#1}}

\newcommand{\refapp}[1]{Appendix~\ref{sec:#1}}

\newcommand{\refeq}[1]{Eqn.~\ref{eq:#1}}

\newcommand{\lblfig}[1]{\label{fig:#1}}
\newcommand{\lblsec}[1]{\label{sec:#1}}
\newcommand{\lbleq}[1]{\label{eq:#1}}

\newcommand{\ignorethis}[1]{}

\newcommand{\myparagraph}[1]{\smallskip \noindent \textbf{#1}}

\def\1{\bm{1}}

\DeclareMathAlphabet{\mathsfit}{\encodingdefault}{\sfdefault}{m}{sl}
\SetMathAlphabet{\mathsfit}{bold}{\encodingdefault}{\sfdefault}{bx}{n}

\newcommand{\Var}{\mathrm{Var}}

\newcommand{\Cov}{\mathrm{Cov}}

\newcolumntype{L}[1]{>{\raggedright\let\newline\\\arraybackslash\hspace{0pt}}m{#1}}
\newcolumntype{C}[1]{>{\centering\let\newline\\\arraybackslash\hspace{0pt}}m{#1}}
\newcolumntype{R}[1]{>{\raggedleft\let\newline\\\arraybackslash\hspace{0pt}}m{#1}}

\newcommand{\ignore}[1]{}

\makeatletter
\DeclareRobustCommand\onedot{\futurelet\@let@token\@onedot}
\def\@onedot{\ifx\@let@token.\else.\null\fi\xspace}

\makeatother

\definecolor{MyDarkBlue}{rgb}{0,0.08,1}
\definecolor{MyDarkGreen}{rgb}{0.02,0.6,0.02}
\definecolor{MyDarkRed}{rgb}{0.8,0.02,0.02}
\definecolor{MyDarkOrange}{rgb}{0.40,0.2,0.02}
\definecolor{MyPurple}{RGB}{111,0,255}
\definecolor{MyRed}{rgb}{1.0,0.0,0.0}
\definecolor{MyGold}{rgb}{0.75,0.6,0.12}
\definecolor{MyDarkgray}{rgb}{0.66, 0.66, 0.66}

\begin{document}
\pagestyle{headings}
\mainmatter
\def\ECCVSubNumber{7231}  %

\title{The Hessian Penalty: A Weak Prior for Unsupervised Disentanglement}

\titlerunning{The Hessian Penalty}
\author{William Peebles\inst{1} \and
John Peebles\inst{2} \and
Jun-Yan Zhu\inst{3} \and \\ Alexei Efros\inst{1} \and Antonio Torralba\inst{4}}
\authorrunning{W. Peebles et al.}

\institute{University of California, Berkeley\inst{1} 
\; \; Yale University\inst{2} \\
Adobe Research\inst{3} \; \; MIT CSAIL\inst{4} \\
}

\maketitle

\begin{abstract}
Existing disentanglement methods for deep generative models rely on hand-picked priors and complex encoder-based architectures. In this paper, we propose the \textit{Hessian Penalty}, a simple regularization term that encourages the Hessian of a generative model with respect to its input to be diagonal. We introduce a model-agnostic, unbiased stochastic approximation of this term based on Hutchinson's estimator to compute it efficiently during training. Our method can be applied to a wide range of deep generators with just a few lines of code. We show that training with the Hessian Penalty often causes axis-aligned disentanglement to emerge in latent space when applied to ProGAN on several datasets. Additionally, we use our regularization term to identify interpretable directions in BigGAN's latent space in an unsupervised fashion. Finally, we provide empirical evidence that the Hessian Penalty encourages substantial shrinkage when applied to over-parameterized latent spaces. We encourage readers to view videos of our disentanglement results on our \href{http://www.wpeebles.com/hessian-penalty}{website} and access our code on \href{https://github.com/wpeebles/hessian_penalty}{GitHub}. %
\end{abstract}

\section{Introduction}
What does it mean to disentangle a function? While Yoshua Bengio has advocated for using ``broad generic priors'' to design disentanglement algorithms~\cite{bengio2013deep}, most recent disentanglement efforts end up being specific to the network architecture used~\cite{DBLP:conf/iclr/HigginsMPBGBML17,kingma2014auto,kim2018disentangling,zhu2018visual} and the types of variation present in datasets~\cite{chen2016infogan,nguyen2019hologan}. 

In this paper, we propose a notion of disentanglement that is simple and general, and can be implemented in a few lines of code (\reffig{pytorch}). Our method is based on the following observation: if we perturb a single component of a network's input, then we would like the \textit{change} in the output to be independent of the other input components. As discussed later, this information is present in the function's Hessian matrix. To encourage a deep neural network to be disentangled, we propose minimizing the off-diagonal entries of the function's Hessian matrix. We call this regularization term a \textit{Hessian Penalty}. Since we can always obtain an estimate of a function's Hessian via finite differences, our method is model-agnostic and requires no auxiliary networks such as encoders.  

We present experiments spanning several architectures and datasets that show applying our Hessian Penalty to generative image models causes the generator's output to become smoother and more disentangled in latent space. We also show that the Hessian Penalty has a tendency to ``turn-off'' latent components, introducing sometimes significant shrinkage in the latent space. We apply our regularization term to BigGAN~\cite{brock2019large} and ProGAN~\cite{karras2018progressive} on ImageNet \cite{deng2009imagenet}, Zappos50K~\cite{finegrained_shoes} and CLEVR~\cite{johnson2017clevr}. We provide quantitative metrics that demonstrate our method induces disentanglement, latent space shrinkage and smoothness compared to baseline models.

\section{Related Work}

\myparagraph{Derivative-based regularization.}
Recently, researchers have proposed regularizing derivatives of various orders to enhance the performance of deep networks. Most notably, Moosavi et al.~\cite{moosavi2019robustness} regularized the eigenvalues of  classifiers' Hessian matrices to improve adversarial robustness. Several works have also explored regularizing derivatives in generative models. StyleGAN-v2~\cite{Karras2019stylegan2} presented a regularization function to encourage the Jacobian matrix of the generator in generative adversarial networks (GANs) ~\cite{goodfellow2014generative} to be orthogonal. Odena et al.~\cite{odena2018generator} introduced a regularization term for clamping the generator's Jacobian's singular values to a target range. To combat mode collapse in image-to-image translation~\cite{isola2017image,zhu2017unpaired}, Yang et al.~\cite{yang2019diversity} proposed a regularization term that encourages changes in the output of the generator to be proportional to changes in latent space; this effectively amounts to preventing the generator's average gradient norm from being degenerate. Ramesh et al.~\cite{ramesh2018spectral} proposed a regularizer for aligning singular vectors of the generator's Jacobian to coordinate axes in latent space to encourage disentanglement. The gradient penalty~\cite{gulrajani2017improved} was proposed to regularize the input gradient of discriminators in GANs.

\myparagraph{Disentanglement in generative models.}
A plethora of prior work on disentangling deep networks focuses on variational autoencoders (VAEs)~\cite{kingma2014auto} with various extensions to the original VAE formulation~\cite{DBLP:conf/iclr/HigginsMPBGBML17,kim2018disentangling,locatello2019challenging,karaletsos2015bayesian,kulkarni2015deep,jha2018disentangling}. Several methods have been proposed to induce disentangled representations in GANs. InfoGAN \cite{chen2016infogan} proposed maximizing the mutual information between an auxiliary latent code and the generator's output. Recent methods have used latent code swapping and mixing for learning disentangled models, coupled with adversarial training~\cite{mathieu2016disentangling,singh2019finegan,hu2018disentangling}. StyleGAN~\cite{karras2019style} introduced a generator architecture that enables control of aspects such as object pose and color. %
Disentanglement of 3D factors of variation has been learned by introducing implicit 3D convolutional priors into the generator~\cite{nguyen2019hologan} or by using explicit differentiable rendering pipelines~\cite{zhu2018visual}. 
Recently, it has been shown that GANs automatically learn to disentangle certain object categories in the channels of their intermediate activations~\cite{bau2019gandissect}; this innate disentanglement can then be leveraged to perform semantic edits on inversions of natural images in the generator's latent space~\cite{bau2019semantic,shen2020interpreting,collins2020editing}.

A line of work has also explored using vector arithmetic in latent space to control factors of variation. For example, DCGAN~\cite{radford2015unsupervised} showed that latent vectors corresponding to specific semantic attributes could be added or subtracted in latent space to change synthesized images in a consistent way. Jahanian et al.~\cite{jahanian2019steerability} learned directions in GANs corresponding to user-provided image transformations. Recently, an unsupervised approach~\cite{voynov2020unsupervised} learns to discover interpretable directions by learning both latent space directions and a classifier to distinguish between those directions simultaneously.

\myparagraph{Independent component analysis.} 
The Hessian Penalty is somewhat reminiscent of Independent Component Analysis (ICA)~\cite{hyvarinen2000independent}, a class of algorithms that tries to ``unmix" real data into its underlying independent latent factors. Recent work has extended nonlinear ICA to modern generative models~\cite{khemakhem2020variational,khemakhem2020ice}, including VAEs and energy-based models~\cite{lecun2006tutorial}. These papers have shown that, under certain conditions, the independent latent factors can be identified up to simple transformations. A simple way to connect our work to ICA is by considering the Hessian Penalty as imposing a prior on the space of possible mixing functions. This prior biases the mixing function to have a diagonal Hessian.%

\section{Method}

We begin by presenting a model-agnostic regularization term that aims to encourage the emergence of disentangled representations. We then discuss how to apply this to generative models.

\subsection{Formulation}
\lblsec{formulation}
Consider any scalar-valued function $G: \mathbb{R}^{|z|} \rightarrow \mathbb{R}$, where $z$ denotes the input vector to $G$ and $|z|$ denotes the dimensionality of $z$. To disentangle $G$ with respect to $z$, we need each component of $z$ to control just a single aspect of variation in $G$; in other words, varying $z_i$ should produce a change in the output of $G$, mostly independently of the other components $z_{j \neq i}$.

Let's consider what this means mathematically. We refer to the Hessian matrix of $G$ with respect to $z$ as $H$. Let's consider an arbitrary off-diagonal term $H_{ij}$ of this Hessian and contemplate what it means if it is equal to zero:

\begin{equation} \lbleq{highschool_calculus}
H_{ij} = \frac{\partial^2 G}{\partial z_i \partial z_j} = \frac{\partial}{\partial z_j} \left( \frac{\partial G}{\partial z_i} \right) = 0.
\end{equation}
Consider the inner derivative with respect to $z_i$ in \refeq{highschool_calculus}. Intuitively, that derivative measures how much $G$'s output changes as $z_i$ is perturbed. If the outer derivative with respect to $z_j$ of the inner derivative is zero, it means that $\frac{\partial G}{\partial z_i}$ is not a function of $z_j$. In other words, $z_j$ has no effect on how perturbing $z_i$ will change $G$'s output.

The above observation gives rise to our main idea. We propose adding a simple regularizer to any function/deep neural network $G$ to encourage its Hessian with respect to an input to be diagonal; we simply minimize the sum of squared off-diagonal terms. We call this regularization function a \textit{Hessian Penalty}:

\begin{equation} \lbleq{reg}
    \mathcal{L}_H(G) = \sum_{i=1}^{|z|} \sum_{j \neq i}^{|z|} H_{ij} ^ 2.
\end{equation}

\subsection{Generalization to Vector-valued Functions}
\lblsec{manytomany}
Of course, most deep networks---such as generative models that synthesize realistic images, video, or text---are \textit{not} scalar-valued functions. A simple way to extend the above formulation to these vector-valued functions is to instead penalize the Hessian matrix of each scalar component in the output of $x = G(z)$ individually, where $x$ denotes the vector of outputs. For brevity, we refer to the length-$|x|$ collection of each $|z|\times|z|$ Hessian matrix as $\mathbf{H}$, where $\mathbf{H}_i$ is the Hessian matrix of $x_i$. Then \refeq{reg} can be slightly modified to:

\begin{equation} \lbleq{reg_many}
    \mathcal{L}_{\mathbf{H}}(G) = \max_i \mathcal{L}_{\mathbf{H}_i}(G),
\end{equation}
where  $\mathcal{L}_{\mathbf{H}_i}$ refers to computing \refeq{reg} with $H = \mathbf{H}_i$. This is a general way to extend the Hessian Penalty to vector-valued functions without leveraging any domain knowledge. In place of the $\max$, we also experimented with taking a mean. We have found that the formulation above imposes a stronger regularization in certain instances, but we have not thoroughly explored alternatives.%

\subsection{The Hessian Penalty in Practice}
\lblsec{pratice}

\begin{figure*}[t]
\begin{center}
\begin{minted}{python}
def hessian_penalty(G, z, k, epsilon):
    # Input G: Function to compute the Hessian Penalty of
    # Input z: Input to G that the Hessian Penalty is taken w.r.t. 
    # Input k: Number of Hessian directions to sample
    # Input epsilon: Finite differences hyperparameter
    # Output: Hessian Penalty loss
    G_z = G(z)
    vs = epsilon * random_rademacher(shape=[k, *z.size()])
    finite_diffs = [G(z + v) - 2 * G_z + G(z - v) for v in vs]
    finite_diffs = stack(finite_diffs) / (epsilon ** 2)
    penalty = var(finite_diffs, dim=0).max()
    return penalty
\end{minted}
\caption{PyTorch-style pseudo-code for the Hessian Penalty. It encourages the change in a neural net's output caused by perturbing one input component to be invariant to the other components, paving the way for disentanglement.}
\lblfig{pytorch}
\end{center}
\end{figure*}

Computing the Hessian matrices in \refeq{reg} and \refeq{reg_many} during training is slow when $|z|$ is large. Luckily, it turns out that we can express \refeq{reg} in a different form which admits an unbiased stochastic approximator:

\begin{equation} \lbleq{approx}
    \mathcal{L}_H(G) = \text{Var}_v \left( v^T H v \right) 
\end{equation}
Where $v$ are Rademacher vectors (each entry has equal probability of being $-1$ or $+1$), and $v^T H v$ is the second directional derivative of $G$ in the direction $v$ times $|v|$. \refeq{approx} can be estimated using the unbiased empirical variance. In practice, we sample a small number of $v$ vectors, typically just two, to compute this empirical variance. If  \refeq{reg} and \refeq{approx} are equal to each other, then minimizing \refeq{approx} is equivalent to minimizing the sum of squared off-diagonal elements in $H$. This result was previously shown by Hutchinson~\cite{hutchinson1989stochastic,avron2011randomized}, but we include a simple proof in Appendix B.

\newtheorem{theorem2}{Theorem}
\begin{theorem} \lbleq{approx_thm}
$\text{Var}_v \left( v^T H v \right) = 2 \sum_{i=1}^{|z|} \sum_{j \neq i}^{|z|} H_{ij} ^ 2$.
\end{theorem}

\begin{proof}
See \refapp{app:proof}.
\end{proof}

One problem still remains: we need to be able to quickly compute the second directional derivative term in \refeq{approx}. We can do this via a second-order central finite difference approximation:

\begin{equation} \lbleq{fd}
    v^T H v \approx \frac{1}{\epsilon^2} \left[ G(z + \epsilon v) - 2 G(z) + G(z - \epsilon v) \right],
\end{equation}
where $\epsilon > 0$ is a hyperparameter that controls the granularity of the second directional derivative estimate. In practice, we use $\epsilon = 0.1$. This approximation enables the Hessian Penalty to work for functions whose analytic Hessians are zero, such as piece-wise linear neural networks.

Figure \ref{fig:pytorch} shows an implementation of \refeq{approx} in PyTorch using the finite difference approximation described in \refeq{fd}; it is only about seven lines of code and can be easily inserted into most code bases.

\begin{figure*}[t]
\begin{center}
\includegraphics[width=1.0\linewidth]{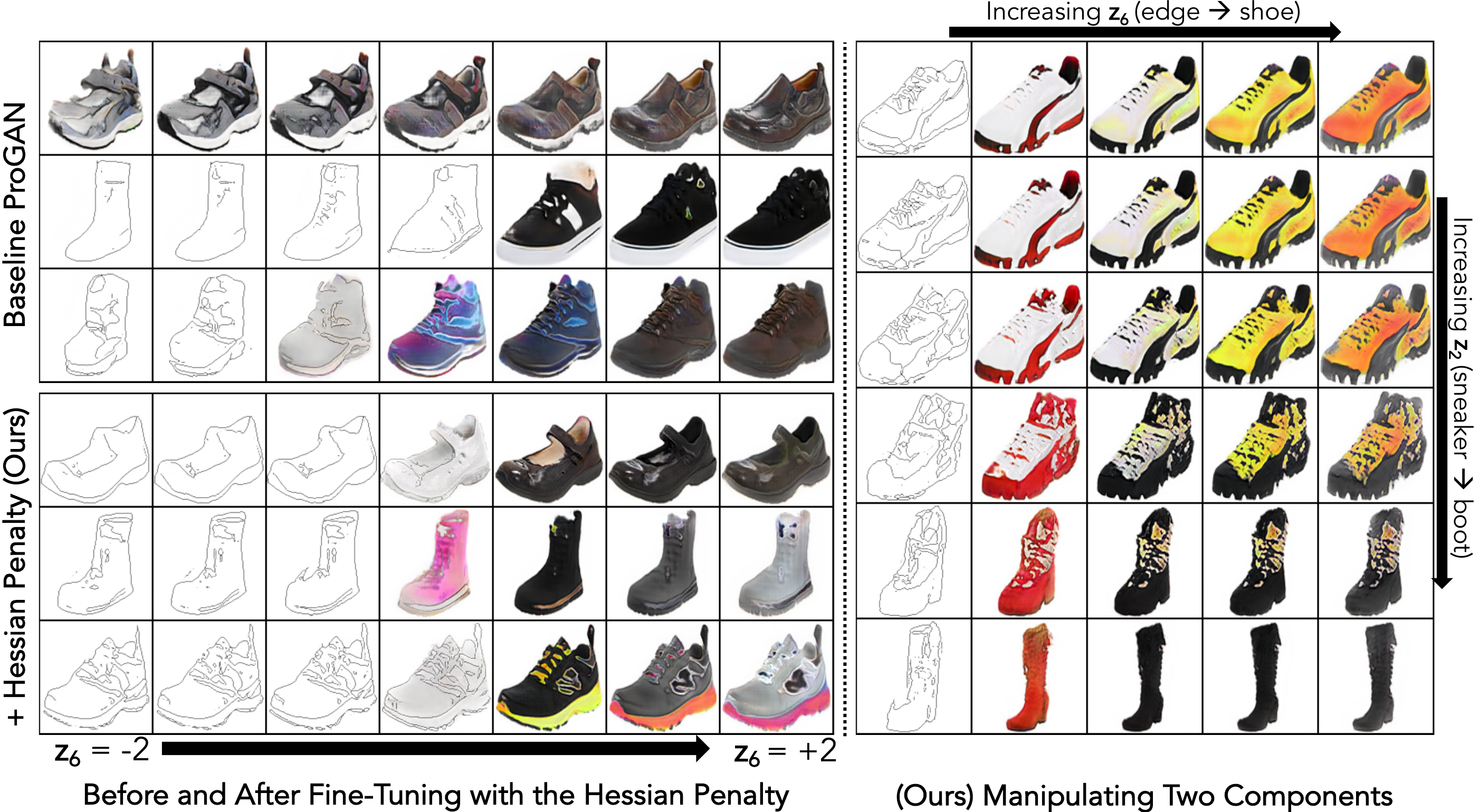}

\caption{The effect of the Hessian Penalty on disentangling the edge $\rightarrow$ shoe factor of variation in unconditional ProGAN trained on Edges$+$Shoes. We sample three 12-dimensional $z$ vectors from a standard Gaussian. Each row corresponds to one of these three vectors. Moving across a row, we interpolate the scalar component $z_6$ from $-2$ to $+2$, leaving the other 11 components fixed. \textbf{Top left:} A ProGAN prior to fine-tuning. It fails to uncover a disentangled $z$ component that controls edge $\rightarrow$ shoe; sometimes the shoe never becomes an edge (first row), and the value of $z_6$ where an edge becomes a shoe is inconsistent. Even when the edge does transform, the resulting shoe barely resembles the edge. \textbf{Bottom left:} Fine-tuning the same ProGAN with our Hessian Penalty. After fine-tuning, edge $\rightarrow$ shoe is cleanly disentangled by $z_6$; edges consistently become shoes right at $z_6 = 0$. For $z_6 > 0$, the component changes the style of the shoe while preserving the structure. \textbf{Right: (ours)} Manipulating two components simultaneously; we sample $z_6$ at $-1$, 0.3, 0.6, 0.9 and 1.2, and $z_2$ regularly between $-2$ and $2$.}

\label{fig:e2s_bl_z6}
\end{center}
\vspace{-0.2in}

\end{figure*}
\myparagraph{Generalization to arbitrary feature spaces. }
In the above description of the Hessian Penalty, $z$ has referred to the input to the function $G$. In general, though, $z$ could be any intermediate feature space of $G$. Similarly, $G$ could refer to any downstream intermediate activation in a generator. In most experiments, we tend to optimize the Hessian Penalty of several intermediate activations with respect to the initial input $z$ vector to achieve a stronger regularization effect.
 
\subsection{Applications in Deep Generative Models}

The above formulation of the Hessian Penalty is model-agnostic; it can be applied to any function without modification. But, here we focus on applying it to generative models. Specifically, we will investigate its applications with generative adversarial networks (GANs)~\cite{goodfellow2014generative}. For the remainder of this paper, $G$ will now refer to the generator and $D$ will refer to the discriminator. GANs are commonly trained with the following adversarial objective, where $x$ now refers to a sample from the real distribution being learned, and $f$ specifies the GAN loss used:

\begin{equation}
    \mathcal{L}_{\text{adv}} = \mathop{\mathbb{E}}_{x \sim p_{\text{data}}(x)}\left[f(D(x))\right] + \mathop{\mathbb{E}}_{z \sim p_z(z)}\left[f(1 - D(G(z)))\right].
\end{equation}

When we apply a Hessian Penalty, the discriminator's objective remains unchanged. The generator's loss becomes:

\begin{equation} \lbleq{gan_hp}
    \mathcal{L}_\text{G} = \underbrace{\displaystyle \mathop{\mathbb{E}}_{z \sim p_z(z)} \left[f\left(1 - D(G(z))\right)\right]}_{\text{Standard Adversarial Loss}} + \underbrace{\lambda \displaystyle \mathop{\mathbb{E}}_{z \sim p_z(z)} \left[ \mathcal{L}_{\mathbf{H}}(G) \right]}_{\text{The Hessian Penalty}},
\end{equation}
where the weight $\lambda$  balances the two terms. Interestingly, we find that fine-tuning a pre-trained GAN with \refeq{gan_hp} in many cases tends to work as well as or better than training from scratch with the Hessian Penalty. This feature makes our method more practical since it can  be used to quickly adapt pre-trained GANs.

\section{Experiments}
\lblsec{expr}

We first demonstrate the effectiveness of the Hessian Penalty in encouraging disentangled representations to emerge in ProGAN across several datasets. Then, we show that the Hessian Penalty can be extended to learn interpretable directions in BigGAN's latent space in an unsupervised manner. For full animated visual results, including comparisons against InfoGAN, please see our \href{http://www.wpeebles.com/hessian-penalty}{website}.

\subsection{ProGAN with Hessian Penalty}
\lblsec{ft_proggan}
We first qualitatively assess how well our Hessian Penalty performs when disentangling an unconditional ProGAN~\cite{karras2018progressive} trained on various datasets. In these experiments, we apply the Hessian Penalty to the first ten out of thirteen convolutions, immediately following pixel normalization layers. Synthesized images are $128\times128$ resolution in all experiments.

\subsubsection{Edges$+$Shoes.} A commonly-used dataset for the problem of image-to-image translation~\cite{isola2017image,zhu2017unpaired} is Edges$\rightarrow$Shoes~\cite{finegrained_shoes}. To see if our method can automatically uncover a $z$ component that performs image-to-image translation without domain supervision, we train an unconditional ProGAN on Edges$+$Shoes, created by mixing all ~50,000 edges and 50,000 shoes into a single image dataset. We then train ProGAN on this mixture of images. 

As seen in Figure~\ref{fig:e2s_bl_z6}, the baseline ProGAN is unable to uncover a component that controls edges$\leftrightarrow$shoes. However, once we fine-tune the ProGAN with the Hessian Penalty, we uncover such a component---$z_6$. When this component is set greater than zero, it produces shoes; when set less than zero, it produces edges. Interestingly, this component is akin to recent multimodal image-to-image translation methods~\cite{zhu2017toward,huang2018multimodal}. As one increases $z_6$ beyond zero, it changes the style and colors of the shoe while preserving the underlying structure. Figure \ref{fig:e2s_bl_z6} also shows how we can leverage this disentanglement to easily manipulate the height of a shoe without inadvertently switching from the edge domain to the shoe domain, or vice versa. The baseline model fails to perform such clean edits. Empirically, we observed that a ProGAN trained with InfoGAN's loss term~\cite{chen2016infogan} failed to uncover disentangled factors of variation of this dataset (see our \href{http://www.wpeebles.com/hessian-penalty}{website} for animated visual comparisons).

\begin{figure*}[t]
\begin{center}
\includegraphics[width=0.8\linewidth]{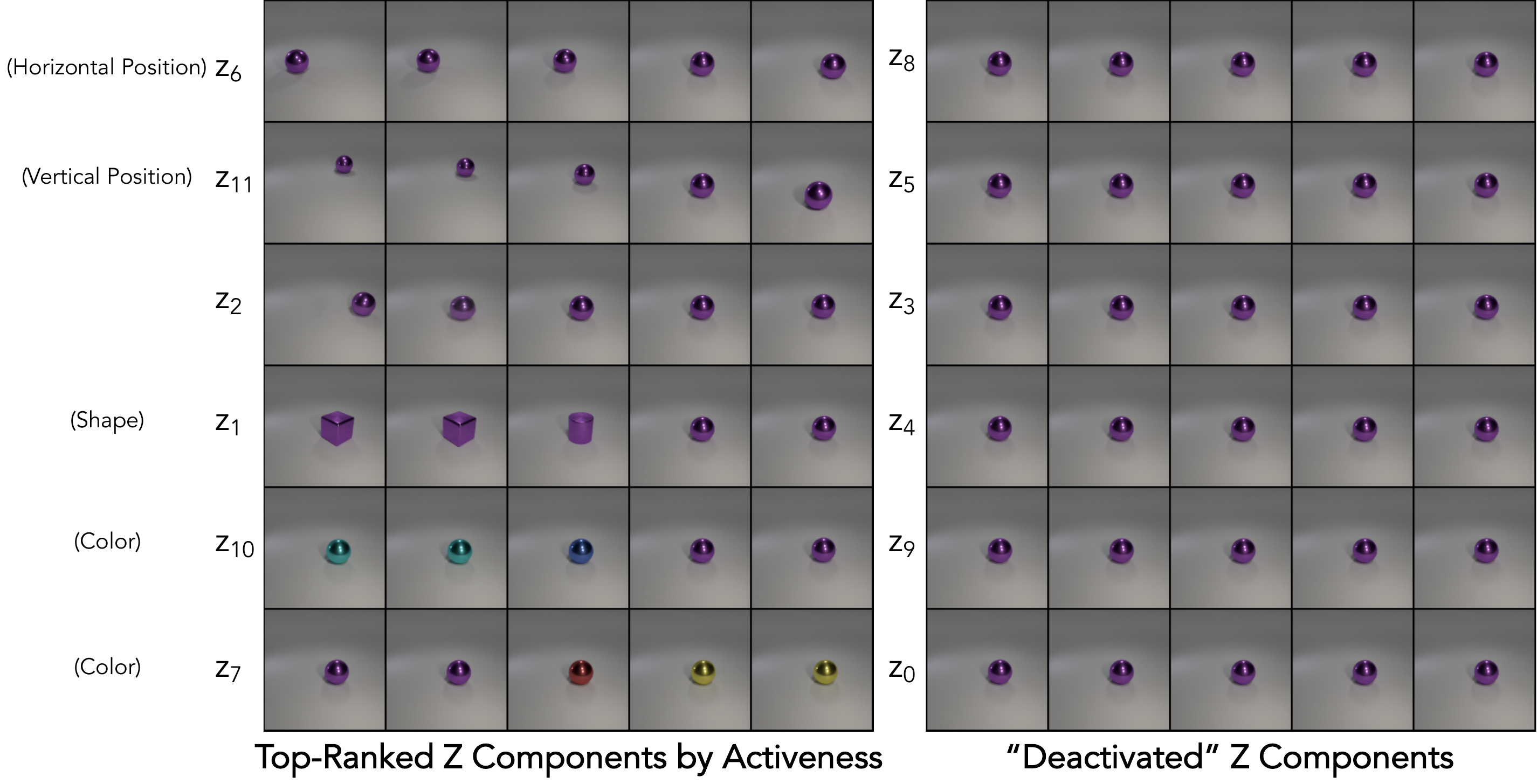}
\caption{All $12$ $z$ components learned by our method on CLEVR-Simple, sorted by their activeness scores (see \ref{sec:shrinkage}). \textbf{Left:} The top six scoring components which uncover color, position, and shape factors of variation. \textbf{Right:} The bottom six scoring components; note that they barely affect the image.}
\label{fig:uncurated}
\end{center}
\end{figure*}

\subsubsection{CLEVR.}

It is difficult to determine if a disentanglement algorithm ``works'' by only testing on real data since the ground truth factors of variation in such datasets are usually unknown and sometimes subjective. As a result, we create three synthetic datasets based on CLEVR~\cite{johnson2017clevr}. The first dataset, CLEVR-Simple, has four factors of variation: object color, shape, and location (both horizontal and vertical). The second, CLEVR-1FOV, features a red cube with just a single factor of variation (FOV): object location along a single axis. The third, CLEVR-Complex, retains all factors of variation from CLEVR-Simple but adds a second object and multiple sizes for a total of ten factors of variation (five per object). Each dataset consists of approximately 10,000 images.

Figure~\ref{fig:uncurated} shows all $12$ $z$ components learned by our method when trained on CLEVR-Simple; we are able to uncover all major factors of variation in the dataset. Figure \ref{fig:two_obj_comp} compares the performance of our method on CLEVR-Complex. Our method does a better job of separating object control into distinct $z$ components. For example, changing $z_{11}$ in the baseline model leads to significant changes in both objects. After fine-tuning with the Hessian Penalty, it mostly---but not entirely---controls the vertical position of the left-most object.

\begin{figure*}[t]
\begin{center}
\includegraphics[width=1.0\linewidth]{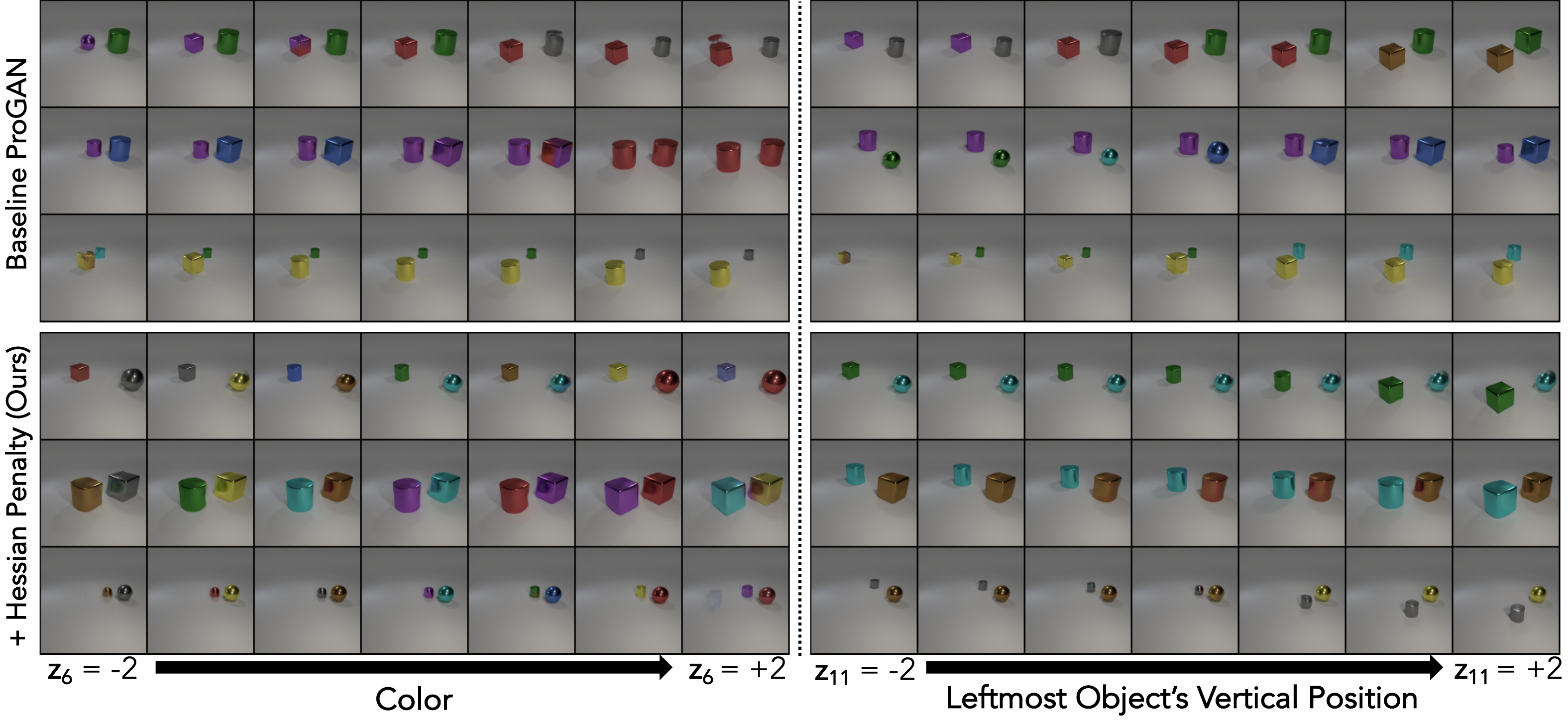}
\caption{We compare the disentanglement quality of a ProGAN with and without our regularization term on CLEVR-Complex. \textbf{Top-Left:} The baseline ProGAN's $z_7$ component does not appear to control a disentangled, interpretable factor of variation. \textbf{Bottom-Left:} After fine-tuning, $z_7$ controls the color of both objects in the scene. Although the component is interpretable, it is not truly disentangled since the color of the objects are two independent factors of variation. \textbf{Top-Right:} In the baseline ProGAN, the latent component $z_{11}$ somewhat controls the vertical position of the left-most object in the scene. However, it significantly alters the appearance of the right-most object. \textbf{Bottom-Right:} After fine-tuning with our Hessian Penalty, $z_{11}$ more cleanly controls vertical movement of the left-most object, although the color of the object still slightly changes sometimes. Note that the right-most object is left mostly unchanged.}
\label{fig:two_obj_comp}
\end{center}
\end{figure*}
\begin{figure*}[t]
\begin{center}
\includegraphics[width=\linewidth]{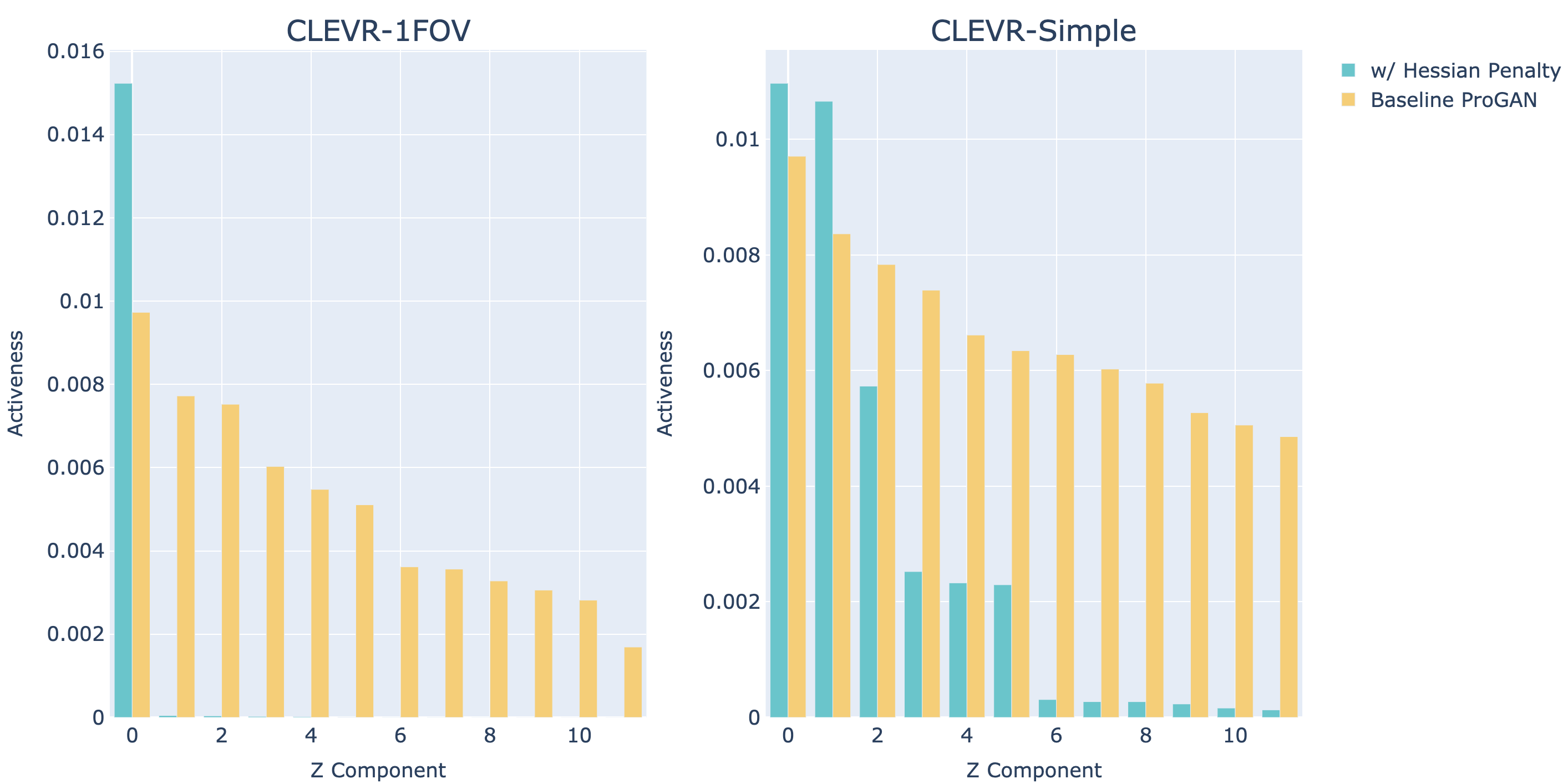}
\vspace{-0.1in}
\caption{\textbf{Latent space shrinkage.} We sort the 12 $z$ components of different generators by their ``activeness'' (how much they control $G$'s output). In baseline ProGAN, $z$ components have somewhat uniform activeness, regardless of the number of factors of variation in the data. When training with the Hessian Penalty on CLEVR-1FOV (one true factor of variation), all $z$ components except one are effectively turned-off. We observe a similar effect in CLEVR-Simple (four factors of variation), where six $z$ components have virtually no control of $G$'s output after being trained with the Hessian Penalty.}
\label{fig:dead}
\end{center}
\vspace{-0.3in}
\end{figure*}

\subsection{Overparameterized Latent Spaces}
\lblsec{shrinkage}

In most circumstances, we do not know a-priori how many factors of variation are in a dataset. Therefore, an ideal disentanglement algorithm would be able to learn a \textit{sparse} representation where it only has $c$ active $z$ components if there are $c$ factors of variation in the dataset. Qualitatively, we observe in several instances that our method ``turns-off'' extra components when its latent space is overparameterized. Figure \ref{fig:uncurated} illustrates this. For CLEVR-Simple, only half of the components produce significant changes in $G(z)$; $G$ effectively collapses on the remaining six components. Is there a way to quantitatively assess the extent to which components get deactivated?

We propose defining the \textit{activeness} of a component $z_i$ as the mean variance of $G(z)$ as we change $z_i$, but leave the other components fixed (where the variance is over $z_i$ and the mean is over all pixels in $G(z)$). We compare the activeness of all 12 components in models trained with and without the Hessian Penalty in Figure \ref{fig:dead}. Indeed, in several instances, training with the Hessian Penalty substantially reduces the activeness of a subset of $z$ components. In particular, our model trained on CLEVR-1FOV effectively deactivates all but one of its $z$ components.

An informal yet intuitive reason to expect that the Hessian Penalty encourages such behavior is by observing that a degenerate solution to the regularization term is for $G$ to completely mode collapse on $z$; if $G$ is no longer a function of $z$, then by definition, its Hessian (including its off-diagonal terms) has all zeros. Similarly, if $G$ collapses on just a subset of $z$ components, then any off-diagonal terms in the Hessian involving those collapsed components will be zero as well.

We also explore what happens when we \textit{underparameterize} the latent space. We train a ProGAN component with $|z| = 3$ on CLEVR-Simple, which has four factors of variation (we refer to this experiment as CLEVR-U). Although it is impossible to fully disentangle CLEVR in this case, we observe that each $z$ component controls $G(z)$ substantially more smoothly after fine-tuning with the Hessian Penalty; we quantitatively assess this in the next section.

 \begin{table}[t]
  {\small
    \centering
    \resizebox{1.0\linewidth}{!}{
 \aboverulesep=0ex
 \belowrulesep=0ex
 \renewcommand{\arraystretch}{1.2}
\begin{tabular}{lcccccccccc}
      \toprule
      \multirow{2}[3]{*}{Method} &
      \multicolumn{2}{c}{Edges$+$Shoes} &
      \multicolumn{2}{c}{CLEVR-Simple} &
      \multicolumn{2}{c}{CLEVR-Complex} &
      \multicolumn{2}{c}{CLEVR-U} &
      \multicolumn{2}{c}{CLEVR-1FOV} \\
      \cmidrule(lr){2-3}
      \cmidrule(lr){4-5}
      \cmidrule(lr){6-7}
      \cmidrule(lr){8-9}
      \cmidrule(lr){10-11}
      & PPL & FID & PPL & FID & PPL & FID & PPL & FID & PPL & FID \\
      \midrule
      
InfoGAN  & 2952.2 & \textbf{10.4} & 56.2 & \textbf{2.9} & 83.9 & \textbf{4.2} & 766.7 & 3.6 & 22.1 & 6.2\\

ProGAN+  & 3154.1 & 10.8 & 64.5 & 3.8 & 84.4 & 5.5 & 697.7 & 3.4 & 30.3 & 9.0\\

ProGAN & 1809.7 & 14.0 & 61.5 & 3.5 & 92.8 & 5.8 & 720.2 & \textbf{3.2} & 35.5 & 11.5\\

w/ HP  & 1301.3 & 21.2 & 45.7 & 25.0 & \textbf{73.1} & 21.1 & 68.7 & 26.6 & 20.8 & \textbf{2.3}\\

w/ HP FT & \textbf{554.1} & 17.3 & \textbf{39.7} & 6.1 & 74.7 & 7.1 & \textbf{61.6} & 26.8 & \textbf{10.0} & 4.5\\
\bottomrule
\vspace{0.1pt}
 \end{tabular}
 }
 }
\caption{Comparing Perceptual Path Lengths (PPL) and Fréchet Inception Distances (FID) for different ProGAN-based methods. Lower is better for both metrics. We compare four different methods: a baseline ProGAN, fine-tuning a ProGAN with the Hessian Penalty (HP FT), training a ProGAN from scratch with the Hessian Penalty (HP) and training a ProGAN with the InfoGAN objective. We also compare against baseline ProGANs that were trained an equal number of iterations as the fine-tuned models (ProGAN$+$). %
We report the model with the best FID within the same number of training iterations, except for ProGAN. 
 PPL and FID was computed with 100,000 and 50,000 samples.} %
\label{tab:ppls}
\vspace{-0.2in}

\end{table}

\subsection{Quantitative Evaluation of Disentanglement}
\lblsec{quant}

Evaluating disentanglement remains a significant challenge~\cite{locatello2019challenging}. Moreover, most existing metrics are designed for methods that have access to an encoder to approximately invert $G$~\cite{karras2019style}. As a result, we report Perceptual Path Length (PPL)~\cite{karras2019style}, a disentanglement metric proposed for methods without encoders:

\begin{equation}
    \text{PPL}(G) = \mathop{\mathbb{E}}_{z^{(1)}, z^{(2)} \sim p_z(z)}\left[\frac{1}{\alpha^2} d\left(G(z^{(1)}), G(\text{slerp}(z^{(1)}, z^{(2)}; \alpha))\right)\right],
\end{equation}
where $z^{(1)}$ and $z^{(2)}$ are two randomly sampled latent vectors, $d(\cdot, \cdot)$ denotes a distance metric, such as LPIPS~\cite{zhang2018perceptual}, and slerp refers to spherical linear interpolation~\cite{white2016sampling}. Intuitively, PPL measures how much $G(z)$ changes under perturbations to $z$; it is a measure of smoothness. %
Given that we are regularizing the Hessian, which controls smoothness, a reasonable question is: ``are we just optimizing PPL?'' The answer is no; since our method only explicitly penalizes off-diagonal components of the Hessian, we are not optimizing the smoothness of $G$ (which is usually defined as being proportional to the maximum eigenvalue of $G$'s Hessian). As a simple counter-example, the one-to-one function $G(z) = \beta z^3$ would trivially satisfy $\mathcal{L}_{\mathbf{H}} = 0$ but could be arbitrarily ``unsmooth''---and thus have large (bad) PPL---for large $\beta$. Empirically, we find that smaller Hessian Penalties do not imply lower PPLs in general. 

Table \ref{tab:ppls} reports PPLs as well as Fréchet Inception Distances (FIDs)~\cite{heusel2017gans}, a coarse measure of image quality. We also compare against a ProGAN trained to maximize the mutual information between a subset of the inputs $z$ vector and the output image, as in InfoGAN~\cite{chen2016infogan}. In general,  our method attains substantially better PPLs compared to other methods across datasets. However, we do sometimes observe a degradation of image quality, especially early in training/fine-tuning. This trait is somewhat shared with $\beta$-VAE-based methods, which essentially trade disentanglement for reconstruction accuracy~\cite{DBLP:conf/iclr/HigginsMPBGBML17}.

\begin{figure*}[t]
\begin{center}
\includegraphics[width=1.0\linewidth]{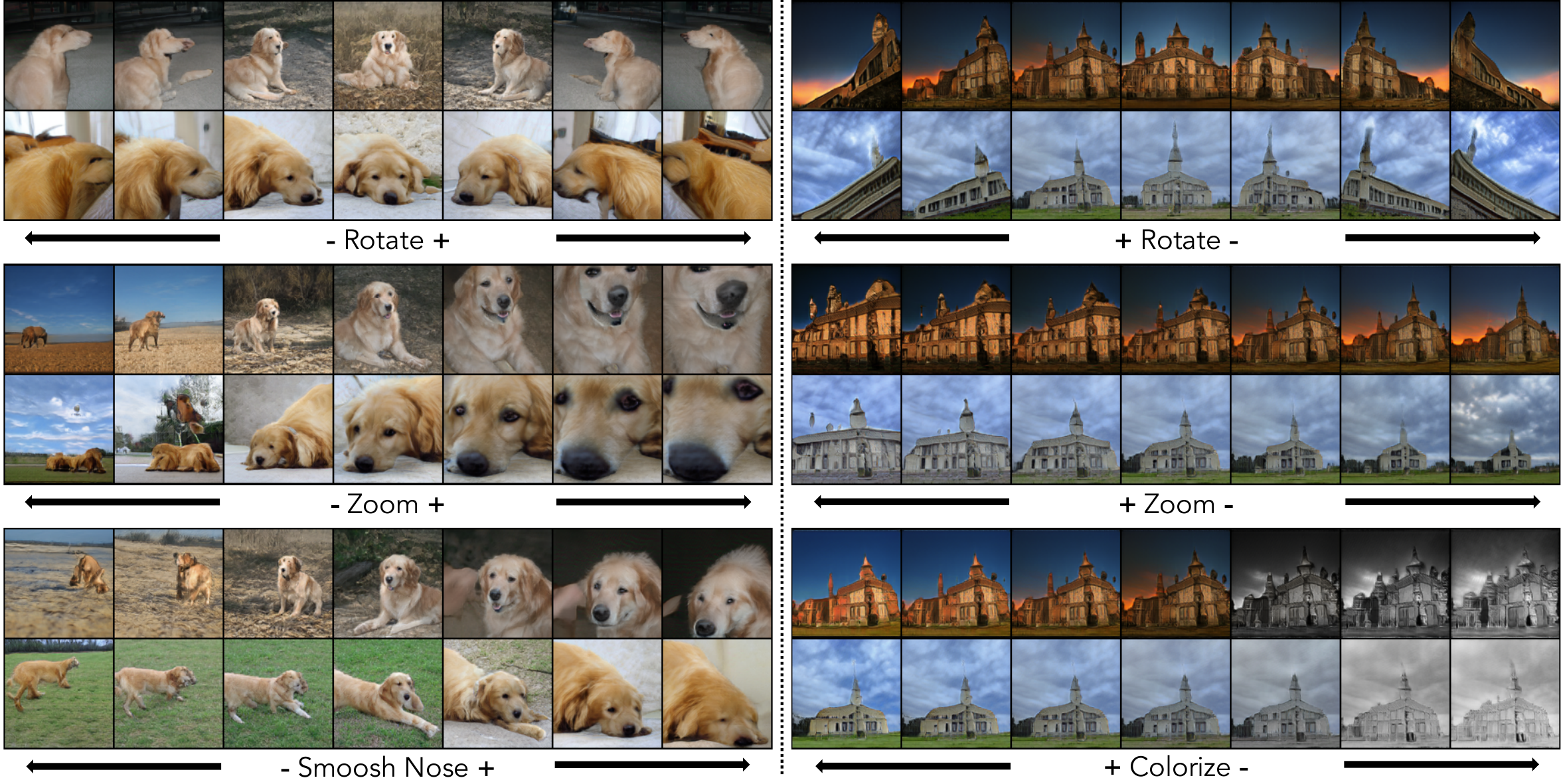}
\vspace{-0.1in}
\caption{Examples of orthonormal directions learned by our method in BigGAN conditioned to synthesize ImageNet Golden Retrievers or Churches. Each factor of variation is a single linear direction in $z$-space. For each direction $Aw_i$ shown, we show $z + \eta Aw_i$, where we linearly move from $\eta=-5$ to $5$ for the golden retriever zooming, church zooming and church colorization directions; we move from $\eta = -3$ to $3$ for the remaining directions. We note that the directions are not completely disentangled; for example, undergoing extreme zooms in Golden Retrievers can sometimes cause substantial changes to the background.}
\label{fig:dogs}
\end{center}
\vspace{-0.3in}

\end{figure*}
\begin{figure*}[t]
\begin{center}
\subfloat{\includegraphics[width=0.1\linewidth]{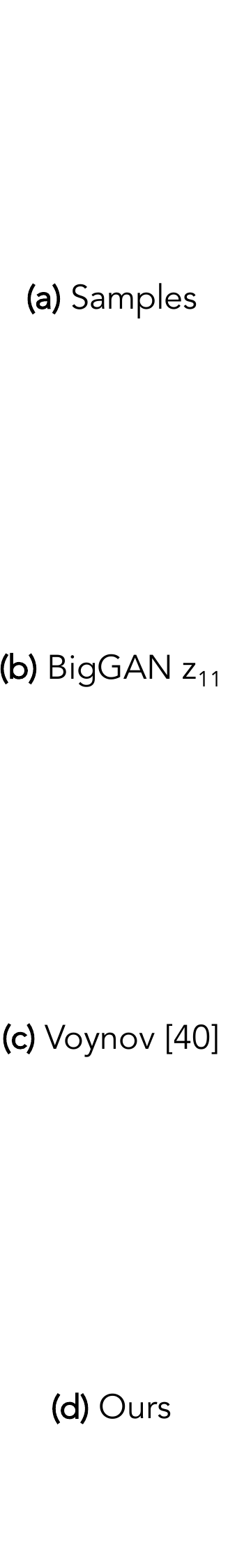}}
\subfloat{\begin{tabular}[b]{c}
\includegraphics[width=0.9\linewidth]{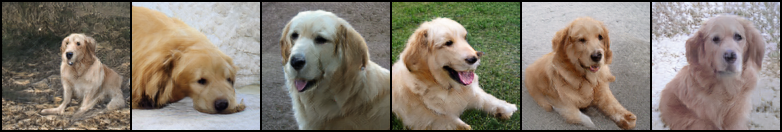}\\
\animategraphics[autoplay,loop,width=0.9\linewidth]{18}{figures/z11/}{001}{036}\\
\animategraphics[autoplay,loop,width=0.9\linewidth]{18}{figures/voynov/}{001}{036}\\
\animategraphics[autoplay,loop,width=0.9\linewidth]{18}{figures/ours/}{001}{036}\end{tabular}}
\vspace{-0.1in}
\caption{(View in Adobe Acrobat) Comparing the quality of latent space edits performed by different methods. (a) We sample six images of golden retrievers from BigGAN. (b) BigGAN's $z_{11}$ component somewhat controls---but does not fully disentangle---rotation; the dogs are unaligned even when their $z$ vectors share the same value of $z_{11}$. (c) We orthogonalize the dogs' $z$ vectors against Voynov and Babenko's \cite{voynov2020unsupervised} learned rotation direction, then add the direction from $\eta = -16$ to $16$. The dogs are unaligned. (d) We can simultaneously align all the dogs by orthogonalizing their $z$ vectors against our learned rotation direction. The alignment is preserved as we add and subtract our direction.}
\label{fig:doggos}
\end{center}
\vspace{-0.4in}
\end{figure*}

\subsection{Unsupervised Learning of Directions in BigGAN}

So far, we have only explored training $G$ with the Hessian Penalty. We can also use the Hessian Penalty to identify interesting directions in $z$-space while leaving $G$'s weights fixed. Recent works have explored learning directions of meaningful variation in $z$-space of pre-trained generators~\cite{jahanian2019steerability,voynov2020unsupervised}. Most notably, Voynov and Babenko~\cite{voynov2020unsupervised} proposed an unsupervised method for learning directions. 

We propose a simple way to learn interesting directions in an unsupervised fashion. Our method begins similarly to prior work~\cite{voynov2020unsupervised}. We instantiate a random orthogonal matrix $A \in \mathbb{R}^{|z|\times N}$, where $N$ refers to the number of directions we wish to learn; the columns of $A$ store the directions we are learning. We then generate images by computing $G(z + \eta Aw_i)$, where $w_i \in \{0,1\}^N$ is a one-hot vector which indexes the columns of $A$ and $\eta$ is a scalar  which controls how far $z$ should move in the direction $Aw_i$. $\eta$ is sampled from a uniform distribution $[-5, 5]$. Our methods diverge at this point. While Voynov and Babenko~\cite{voynov2020unsupervised} simultaneously learn a randomly-initialized regression and classification network to reconstruct $Aw_i$ and $\eta$ from $G(z)$ and $G(z + \eta Aw_i)$, we directly optimize $A^*=\arg \min_A \mathbb{E}_{z, w_i, \eta} \; \mathcal{L}_{\mathbf{H}}(G(z + \eta Aw_i))$, where the Hessian Penalty is now taken w.r.t. $w_i$ instead of $z$.
Intuitively, this amounts to trying to disentangle the columns of $A$. Importantly, we \textit{only} backprop the gradients into $A$, while $G$ is frozen throughout learning. There are no additional loss terms beyond the Hessian Penalty. We set $N=|z|$ in our experiments and restrict $A$ to be orthogonal by applying Gram-Schmidt during each forward pass.

We apply our discovery method to class-conditional BigGAN~\cite{brock2019large} trained on ImageNet~\cite{deng2009imagenet}. We perform two experiments; learning directions when BigGAN is restricted to producing (1) golden retrievers and (2) churches. Figure \ref{fig:dogs} shows our results. We are able to uncover several interesting directions, such as rotation, zooming and colorization. In Figure \ref{fig:doggos}, we compare our learned rotation direction to the one learned by Voynov and Babenko \cite{voynov2020unsupervised}, as well as a similar axis-aligned direction $z_{11}$ in BigGAN. Our direction does a better job at preserving alignment than both baselines. Empirically, we found that some of BigGAN's $z$ components, which are injected directly into deeper layers in $G$, already achieve a reasonable amount of axis-aligned disentanglement; they control aspects such as lighting, color filtering and background removal. We did not observe that our learned directions control these factors as well as certain individual $z$ components. 

One surprising trend we observed was that some of our more interesting and effective components---such as the rotation one---consistently emerged within the first several columns of $A$ across different initializations of $A$. 

 \begin{table}[t]
  {\small
    \centering
    \resizebox{1.0\linewidth}{!}{
 \aboverulesep=0ex
 \belowrulesep=0ex
 \renewcommand{\arraystretch}{1.2}
\begin{tabular}{lcccccccccc}
      \toprule
      \multirow{2}[3]{*}{Method} &
      \multicolumn{2}{c}{Edges$+$Shoes} &
      \multicolumn{2}{c}{CLEVR-Simple} &
      \multicolumn{2}{c}{CLEVR-Complex} &
      \multicolumn{2}{c}{CLEVR-U} &
      \multicolumn{2}{c}{CLEVR-1FOV} \\
      \cmidrule(lr){2-3}
      \cmidrule(lr){4-5}
      \cmidrule(lr){6-7}
      \cmidrule(lr){8-9}
      \cmidrule(lr){10-11}
      & $D_{\%}$ & $D_R$ & $D_{\%}$ & $D_R$ & $D_{\%}$ & $D_R$ &  $D_{\%}$ & $D_R$ &  $D_{\%}$ & $D_R$\\
      \midrule
      
InfoGAN & 63.2\% & 1.6 & 74.7\% & 2.2 & 66.0\% & 1.6 & 74.2\% & 2.2 & 89.9\% & 5.8\\

ProGAN+ & 64.9\% & 1.6 & 69.3\% & 1.5 & 63.8\% & 1.5 & 75.4\% & 2.2 & 78.7\% & 1.5\\

ProGAN & 66.0\% & 1.7 & 69.8\% & 1.6 & 63.8\% & 1.5 & 75.5\% & 2.3 & 79.3\% & 1.7\\

w/ HP & 80.8\% & \textbf{7.3} & 79.3\% & \textbf{6.0} & 82.7\% & 2.5 & 78.9\% & 2.0 & \textbf{92.8\%} & \textbf{61.3}\\

w/ HP FT & \textbf{81.0\%} & 4.5 & \textbf{85.7\%} & 5.6 & \textbf{85.4\%} & \textbf{3.0} & \textbf{81.8\%} & 2.2 & 87.9\% & 14.6\\
\bottomrule
\vspace{0.1pt}
 \end{tabular}
 }
 }
\caption{The empirical effect of the Hessian Penalty to strengthen the Hessian's diagonal. We compute the percentage of Hessian matrices whose max elements lies on the diagonal ($D_{\%}$). We also measure the ratio between the average magnitude of diagonal elements versus off-diagonal elements ($D_R$).} %
\label{tab:diags}
\vspace{-0.2in}

\end{table}

\subsection{What is the Hessian Penalty Actually Doing?}

The Hessian Penalty is designed to encourage the Hessian of generative models to be diagonal. Does this hold up in practice?

If the Hessian Penalty is working as expected, then the relative weight of the diagonal to the off-diagonal components of the Hessian should increase when we apply our method. We propose two ways to measure this. First, we generate 100 fake images at $128\times128$ resolution, and compute the Hessian matrix for each pixel in all of these images (estimated via second-order finite differences). For each of these Hessians, we measure two quantities: (1) if the largest element in that Hessian lies on the diagonal; (2) the relative magnitude between elements on the diagonal versus off-diagonal. For each quantity, we average results over all $100 \cdot 128 \cdot 128 \cdot 3 = 4,915,200$ Hessian matrices; see Table \ref{tab:diags}. Under these two metrics, we find that the Hessian Penalty always strengthens the diagonal, with the exception of CLEVR-U which has an underparameterized latent space and small $3\times3$ Hessians. As an aside, this is somewhat interesting since in all ProGAN experiments, we never explicitly regularize the Hessian of the pixels directly. We only regularize the Hessian of intermediate activations, three or more convolutions before pixels. We also present visualizations of the Hessian matrices themselves; see \refapp{app:vis}.

\section{Discussion}

In this paper, we presented the Hessian Penalty, a simple regularization function that encourages the Hessian matrix of deep generators to be diagonal. When applied to ProGAN trained on Edges$+$Shoes, our method is capable of disentangling several significant factors of variation, such as edges$\leftrightarrow$shoes, a component reminiscent of multimodal image-to-image translation~\cite{zhu2017toward,huang2018multimodal}. When trained on synthetic CLEVR datasets, our method can also uncover the known factors of variation while shrinking the overparameterized latent space. %
We also showed that our method discovers interesting factors such as rotation, colorization and zooming, in the latent space of BigGAN. 

\myparagraph{Limitations.} Although the Hessian Penalty works well for several datasets discussed above, our method does exhibit several notable failure modes.  First, given that the Hessian Penalty is such a weak prior, we cannot expect to obtain perfect disentanglement results. For example, on CLEVR-Complex, our method learns to control the color of both objects---which are independent in the dataset---with a single component. Second, when fine-tuning a generator's weights with our method, image quality can sometimes degrade. Third, computing the Hessian Penalty only at early layers in the network can lead to a degenerate solution where the generator substantially reduces the Hessian Penalty seemingly without any effects on disentanglement or latent space shrinkage. We found that this degeneracy can be mitigated by also applying the Hessian Penalty to deeper intermediate layers and immediately after normalization layers. Nonetheless, given the simplicity of our method, the Hessian Penalty could be a small step towards the grand goals outlined by Yoshua Bengio.

\myparagraph{Acknowledgments.} We thank Pieter Abbeel, Taesung Park, Richard Zhang, Mathieu Aubry, Ilija Radosavovic, Tim Brooks, Karttikeya Mangalam, and all of BAIR for valuable discussions and encouragement. This work was supported, in part, by grants from SAP, Adobe, and Berkeley DeepDrive.
\clearpage
\bibliographystyle{splncs04}
\bibliography{egbib}

\appendix

\clearpage
\section{Implementation Details}
\lblsec{app:implementation}
Figure 1 in the main paper presents the basic implementation of the Hessian Penalty. Here, we include more implementation details for each experiment below:

\subsection{ProGAN Experiments}
\lblsec{app:progan}

When fine-tuning ProGAN, we find that it is critical to gradually phase-in the Hessian Penalty by ramping-up the loss weight at iteration $t$ (call it $\lambda_t$) over $T$ gradient steps:

\begin{equation}
    \lambda_t = \lambda\min\left(1, \frac{t}{T}\right)
\end{equation}

In our experiments, $T=10^5$ or $T=10^6$ tend to work well across datasets. Without this warm-up, the generator sometimes collapses with the additional regularization term imposed on it. Even with the linear warm-up, we commonly observe a sudden decrease in image quality (as well as a sudden spike in Fréchet Inception Distance (FID)) at the beginning of fine-tuning, usually followed by a quick recovery. Perceptual Path Length (PPL) also rapidly decreases at the start of fine-tuning. Increasing $T$ decreases the spikes in both FID and PPL. When training from scratch with the Hessian Penalty, such a warm-up is not needed.

We apply the Hessian Penalty to the generator's intermediate activations instead of its direct output in pixel space. Specifically, we regularize the first 10 convolutions in the network following pixel normalization. In all experiments, we set the finite differences hyperparameter $\epsilon = 0.1$ (used in \refeq{fd}) and sampled $k=2$ Rademacher vectors to compute the variance term in the Hessian Penalty. During fine-tuning, we also use the pre-trained weights of the discriminator. 

We set $\lambda=0.1$ for the Edges$+$Shoes, CLEVR-Simple and CLEVR-Complex Hessian Penalty experiments. We used $\lambda=0.025$ for the CLEVR-1FOV and CLEVR-U experiments because we found using $\lambda=0.1$ led to excessive degradation of image quality at the beginning of fine-tuning for these models.

For the InfoGAN baseline, we add a second head to the discriminator, as done in the original paper. We optimize a mean squared error between the head's output and a portion of the $z$ vector that generated the corresponding image. We use a loss weighting of $1.0$. To strengthen the InfoGAN baseline in the CLEVR experiments, we only regularize the first $c$ components of the $z$ vector, where $c$ is the true number of factors of variation. We keep all other hyperparameters identical.

\subsection{BigGAN Experiments}
\lblsec{app:biggan}
To learn the matrix of directions $A$, we optimize the Hessian Penalty in pixel space of $G(z+\eta  A w_i)$ w.r.t. $w_i$, where $w_i$ is a one at index $i$ and zeros elsewhere, $i$ is uniformly sampled and $\eta \sim \mathcal{U}[-5,5]$, as done in prior work~\cite{voynov2020unsupervised}. We also tried applying the Hessian Penalty to intermediate activations of $G$, but we did not observe a significant difference in results. We use a mean in place of max as the reduction operation for the Hessian Penalty, but we did not ablate the choice. We restrict $A$ to be orthogonal by applying Gram-Schmidt at the beginning of every forward pass. We use Adam~\cite{kingma2014adam} with default parameters and a learning rate of $0.01$.

\subsection{Avoiding Degenerate Solutions}
\lblsec{app:avoid}
In certain cases, applying the Hessian Penalty immediately following a linear operation (such as a convolution) could lead to a degenerate solution where the filter's norm decreased substantially without achieving any disentanglement. This is why we tend to apply the Hessian Penalty after normalization layers. We note that this is only a potential problem when applying the Hessian Penalty to intermediate activations; it is fine to regularize immediately following a linear function if it is the output of the network.

\section{Proof of Theorem~\ref{eq:approx_thm}}
\lblsec{app:proof}
Let $\mathcal{L}_{H} \triangleq \Var(x^\intercal H x)/2$. In practice, we estimate this variance with a finite number of Rademacher vectors $x$. We now prove that $\mathcal{L}_{H}$ is an unbiased estimator of $\sum_{i \neq j} H_{ij}^2$.
\allowdisplaybreaks[4]
\begin{align*}
\Var(x^\intercal H x) &= \Var \left(\sum_{i,j}  H_{ij} x_i x_j \right) \\
&= \Var \left( \sum_i H_{ii} x_i^2 + \sum_{i \neq j}  H_{ij} x_i x_j \right) \\
&= \Var \left( \underbrace{\sum_i H_{ii}}_{\text{constant}} + \sum_{i \neq j}  H_{ij} x_i x_j \right) \\
&= \Var \left(\sum_{i \neq j}  H_{ij} x_i x_j \right)\\
&= \Var \left(2 \sum_{i < j}  H_{ij} x_i x_j \right) \\
&= 4 \Var \left(\sum_{i < j}  H_{ij} x_i x_j \right) \\
&= 4 \Cov \left(\sum_{i < j}  H_{ij} x_i x_j,\sum_{i < j}  H_{ij} x_i x_j \right) \\
&= 4 \sum_{i<j \text{ and } k < \ell} \Cov \left( H_{ij} x_i x_j, H_{k\ell} x_k x_\ell \right) \\
&= 4 \sum_{i<j \text{ and } k < \ell \text{ where } i=k \text{ and } j=\ell } \Cov \left( H_{ij} x_i x_j, H_{k\ell} x_k x_\ell \right)\\ %
&= 4 \sum_{i<j} \Cov \left( H_{ij} x_i x_j, H_{ij} x_i x_j \right)  \\
&= 4 \sum_{i<j} \Var \left( H_{ij} x_i x_j \right)  \\
&= 4 \sum_{i<j}  H_{ij}^2 \Var \left(x_i x_j \right)  \\
&= 4 \sum_{i<j}  H_{ij}^2 \cdot 1  \\
&= 4 \sum_{i<j} H_{ij}^2 \\
&= 2 \sum_{i \neq j} H_{ij}^2
\end{align*}
\allowdisplaybreaks[0]

The above is true if we can show that $i \neq k$ or $j \neq \ell$ implies that
\[
\Cov(H_{ij} x_i x_j, H_{k\ell} x_k x_\ell) =0.
\]

We can do this by noting
\begin{align*}
\Cov(H_{ij} x_i x_j, H_{k\ell} x_k x_\ell) &= H_{ij} H_{k\ell} \Cov(x_i x_j, x_k x_\ell)
\end{align*}

If $i \neq k$ or $j \neq \ell$ than $x_i x_j$ and $x_k x_\ell$ are independent and the covariance of two independent random variables is always $0$. The reason they are independent is that when $i \neq k$ or $j \neq \ell$, it must be the case that either $k \neq i,j$ or $\ell \neq i,j$. In other words, there is a random variable present in the second product that is not present in the first. Since all the entries of $x$ are independent and $\pm 1$, the presence of this random variable in the second product causes it to be independent of the first. In other words, regardless of the value of $x_i x_j$, the presence of some $x_z$ for $z \neq i,j$ in the second product randomizes whether the second product is plus or minus one, independently of the value of $x_i x_j$. Thus, the distribution of $x_k x_\ell$ does not change even if we condition on the value of $x_i x_j$, which is one of the equivalent definitions of independence for random variables. \qedsymbol

\section{Visualizing Pixel Hessians}
\lblsec{app:vis}
One way to visualize the Hessian matrices of $G$ is to select a single, scalar pixel in $G$'s output and compute its Hessian matrix with second-order finite differences. For each dataset below, we sample six images per method and show the Hessians of the pixels with the largest Hessian Penalty in each image.

\begin{figure*}
\begin{center}
\includegraphics[width=\linewidth]{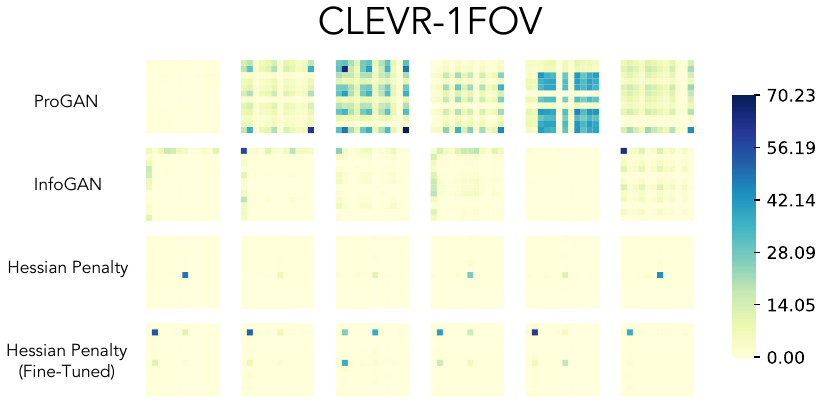}
\caption{Visualizing the $12\times12$ Hessian matrices of CLEVR-1FOV ($|z| = 12$). We compare baseline ProGAN, InfoGAN, training from scratch with Hessian Penalty, and fine-tuning with Hessian Penalty.}
\end{center}
\end{figure*}
\begin{figure*}
\begin{center}
\includegraphics[width=\linewidth]{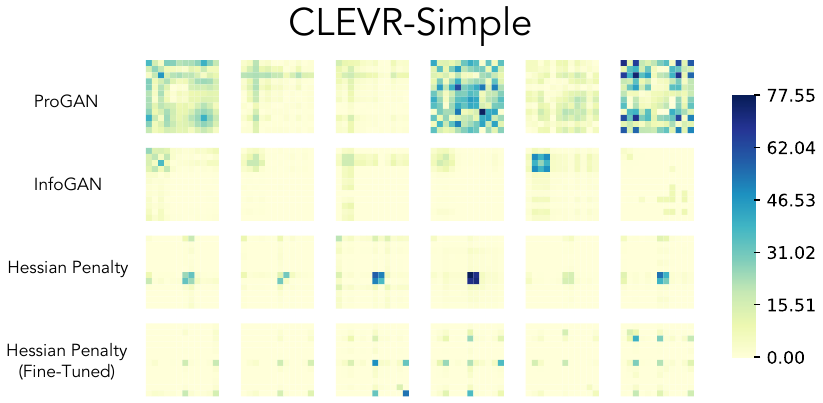}
\caption{Visualizing the $12\times12$ Hessian matrices of CLEVR-Simple ($|z| = 12$). We compare baseline ProGAN, InfoGAN, training from scratch with Hessian Penalty, and fine-tuning with Hessian Penalty.}
\end{center}
\end{figure*}
\begin{figure*}
\begin{center}
\includegraphics[width=\linewidth]{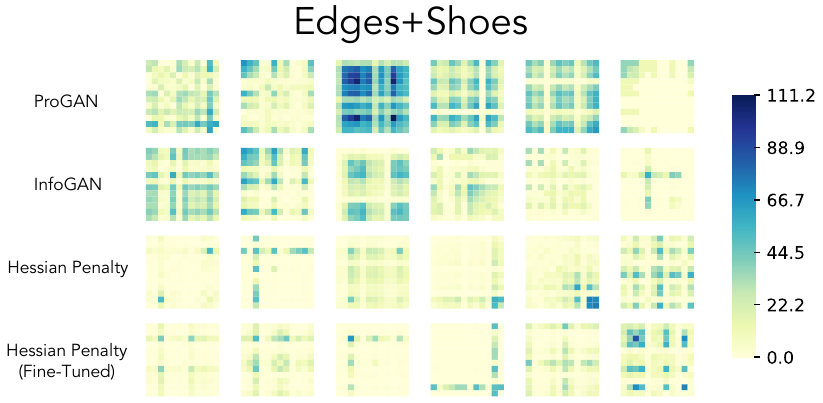}
\caption{Visualizing the $12\times12$ Hessian matrices of Edges$+$Shoes ($|z| = 12$). We compare baseline ProGAN, InfoGAN, training from scratch with Hessian Penalty, and fine-tuning with Hessian Penalty.}
\end{center}
\end{figure*}
\begin{figure*}
\begin{center}
\includegraphics[width=\linewidth]{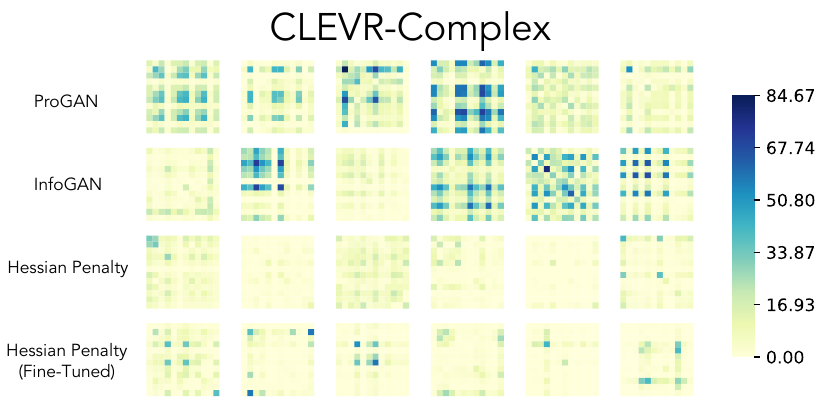}
\caption{Visualizing the $12\times12$ Hessian matrices of CLEVR-Complex ($|z| = 12$). We compare baseline ProGAN, InfoGAN, training from scratch with Hessian Penalty, and fine-tuning with Hessian Penalty.}
\end{center}
\end{figure*}
\begin{figure*}
\begin{center}
\includegraphics[width=\linewidth]{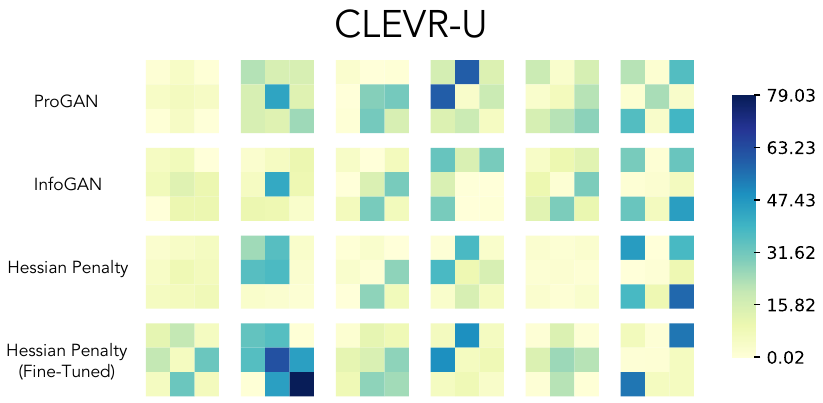}
\caption{Visualizing the $3\times3$ Hessian matrices of CLEVR-U ($|z| = 3$). We compare baseline ProGAN, InfoGAN, training from scratch with Hessian Penalty, and fine-tuning with Hessian Penalty.}
\end{center}
\end{figure*}

\section{Hessian Penalty Heatmaps}

For each image below, we overlay a heatmap of per-pixel Hessian Penalties. These Hessian Penalties were computed in pixel space.

\begin{figure*}[!]
\begin{center}
\includegraphics[width=\linewidth]{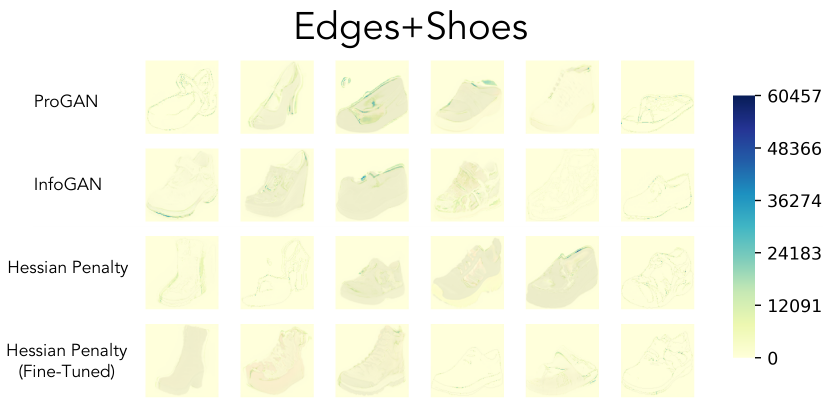}
\caption{Visualizing per-pixel values of the Hessian Penalty in Edges$+$Shoes models.}
\end{center}
\end{figure*}

\end{document}